\documentclass[11pt,a4paper]{article}
\usepackage[utf8]{inputenc}
\usepackage[T1]{fontenc}
\usepackage{amsmath, amssymb, amsthm}
\usepackage{hyperref}
\usepackage{booktabs}
\usepackage{algorithm}
\usepackage{algpseudocode}
\usepackage{graphicx}
\usepackage[a4paper,hmargin=2.5cm,vmargin=2.5cm,verbose]{geometry}

\newtheorem{theorem}{Theorem}
\newtheorem{definition}{Definition}
\newtheorem{proposition}{Proposition}

\newtheorem{remark}{Remark}
\newtheorem{assumption}{Assumption}

\newcommand{\G}{\mathcal{G}}
\newcommand{\Gp}{\mathcal{G}'}
\newcommand{\Hb}{\mathcal{H}}
\newcommand{\V}{\mathcal{V}}
\newcommand{\E}{\mathcal{E}}
\newcommand{\Loss}{\mathcal{L}}
\newcommand{\Surg}{\mathcal{S}}

\begin{document}

\title{Exact Network Surgery: Functional Invariance and Gradient Plasticity in Reactive Computational Graphs}
\author{Abdallah Khemais \\ \textit{ISITCOM, University of Sousse}}
\date{July 2026}

\maketitle

\begin{abstract}
Function-preserving network growth techniques such as Net2Net \cite{net2net} and progressive stacking \cite{stacking} expand a model's capacity without destroying its learned function, but existing formulations either tolerate small numerical perturbations or require a full rebuild of the training program, discarding compiled artifacts and complicating optimizer-state management. We formalize \emph{Exact Network Surgery}: the in-place insertion of a residual block into a live computational graph such that (i) the network function is preserved --- bit-exactly under explicitly stated floating-point hypotheses --- and (ii) the inserted parameters remain trainable immediately after insertion. We prove an identity-morphism theorem for gated residual blocks, a structural-locality theorem showing that a reactive invalidation engine recomputes exactly the downstream cone of the insertion point --- leaving every other node's value \emph{and} optimizer state untouched --- and an escape-from-initialization proposition showing that the \emph{Gradient Shadowing} gate $\alpha$, initialized at zero over a randomly initialized residual branch, receives a generically non-zero gradient at insertion time. We identify and analyze a degenerate configuration --- zero-initialized output projections combined with a zero gate --- which constitutes an exact saddle point from which gradient descent cannot escape, and explain why our construction avoids it. Every claim above is validated empirically on the reference implementation in \emph{NeuroDSL} \cite{neurodsl}, a reactive define-and-run graph engine in Julia: grafting is bit-exact on every logit tested (0 mismatches out of 1600, preconditions verified rather than assumed); the gate escapes zero at the first optimizer step and unlocks the branch gradients at the second, exactly as predicted; the degenerate configuration exhibits gradients that are identically zero --- not merely small --- for the entire 600-step run; surgery cost tracks the downstream cone size with $r = 0.9992$ while the graft-plus-invalidation bookkeeping itself is constant ($\sim$0.75~ms) across insertion depths; and training resumes bit-identically across a real process restart. A clearly-flagged preliminary appendix reports first single-seed observations on what happens to a grafted gate over the training that follows insertion, including a learning-rate ladder in which the gate magnitude and the grafted branch's norm move in opposite directions, so that the gate alone reverses the intended ordering of how far a graft has functionally drifted from the identity.
\end{abstract}

\section{Introduction}

Growing a neural network during training --- inserting new layers into a model that has already learned --- is attractive whenever training a large model from scratch is wasteful: curriculum-style capacity scheduling \cite{stacking, bert2bert}, continual learning, and architecture search all benefit from it. The central difficulty is not conceptual but operational: the insertion must not destroy the learned function, and the surrounding training machinery (optimizer state, compiled kernels, cached activations) must survive the mutation.

Eager frameworks such as PyTorch \cite{pytorch} are define-by-run and therefore support topological mutation natively: one can assign a new module to an attribute and the next forward pass will use it. The cost of surgery in such frameworks is not the mutation itself but its \emph{consequences}: (i) under compilation (\texttt{torch.compile}, \texttt{jax.jit} \cite{jax}), any change to the traced program triggers re-tracing and recompilation of the whole graph; (ii) optimizer state must be re-associated with a mutated parameter set by hand, a common source of silent bugs; and (iii) there is no notion of \emph{partial validity} --- the framework cannot distinguish the subgraph whose cached values are still correct from the subgraph invalidated by the mutation.

\emph{NeuroDSL} \cite{neurodsl} takes a different approach: the computational graph is a persistent, reactive DAG. Nodes own their values and their optimizer state; edges carry explicit dependency information; and a mutation triggers an $O(|\text{affected}|)$ invalidation wave rather than a global rebuild. This paper shows that, on top of such an engine, network surgery can be made \emph{exact} --- the mutated network is functionally indistinguishable from the original, under hypotheses we make precise --- while remaining \emph{plastic}: the new parameters start learning at the very first post-surgery optimization step.

\paragraph{Contributions.}
\begin{enumerate}
    \item A formal definition of grafting on computational graphs and of functional exactness (Section~\ref{sec:prelim}).
    \item An identity-morphism theorem for gated residual blocks, together with a qualified bit-exactness proposition that states its IEEE-754 hypotheses explicitly (Section~\ref{sec:exact}).
    \item A structural-locality theorem: reactive invalidation recomputes exactly the downstream cone of the insertion point, proved by double inclusion over an explicit algorithm, so that parameters outside the cone keep their values and their optimizer state (Section~\ref{sec:locality}). The same theorem makes a \emph{structural} undo of a graft cheap and exact at $O(|\V_s^+|)$ log size, regardless of how far training has since moved the grafted function (Section~\ref{sec:undo}).
    \item \emph{Gradient Shadowing}, a gating scheme in the ReZero family \cite{rezero}, with a proof that the gate escapes its zero initialization in one step, and an analysis of the degenerate double-zero configuration that must be avoided (Section~\ref{sec:shadow}).
    \item An empirical validation of every claim above (Section~\ref{sec:experiments}): bit-exactness with verified preconditions (E1), the three-arm plasticity experiment confirming both the escape and the saddle point (E2), cost-versus-depth measurements confirming locality (E3), bit-identical training continuation across a real process restart (F3), and a pre-registered utility study at equal compute. A clearly-flagged preliminary appendix (Appendix~\ref{sec:appendix-gate}) reports single-seed observations on post-insertion gate dynamics: an exact functional undo-distance identity, a controlled early-vs-late timing comparison (E6), and a warm-start counterfactual tracing most of the apparent timing effect to an optimizer-moment artifact (E8).
\end{enumerate}

\section{Related Work}
\label{sec:related}

\paragraph{Function-preserving growth.} Net2Net \cite{net2net} introduced function-preserving transformations (\emph{Net2WiderNet}, \emph{Net2DeeperNet}) based on identity-initialized layers; exactness holds for ReLU networks but is perturbed by normalization layers and requires care with weight symmetry (broken in practice by added noise). Progressive stacking \cite{stacking} and bert2BERT \cite{bert2bert} grow Transformers by duplicating layers, accepting a functional perturbation that is repaired by continued training. LEMON \cite{lemon} restores lossless expansion for LayerNorm-based Transformers. GradMax \cite{gradmax} grows networks by maximizing the gradient norm of the new weights, which is the dual of our concern: we guarantee the new weights have a \emph{usable} gradient path, GradMax optimizes its magnitude. Concurrent work grows a decoder-only Transformer by stacking new layers above a \emph{frozen} pretrained substrate \cite{frozensubstrate} --- the same configuration as the appendix experiment reported here --- but does not claim function preservation across the expansion; exactness at the moment of insertion, and the plasticity that has to survive it, are what the present paper adds to that setting.

\paragraph{Zero-gated residuals.} Initializing a residual branch so that it initially contributes nothing is a well-studied stabilization device: Fixup \cite{fixup} zero-initializes the last layer of each residual branch; ReZero \cite{rezero} multiplies each branch by a learned scalar $\alpha$ initialized at $0$ over \emph{randomly initialized} weights; LayerScale \cite{layerscale} uses a learned diagonal gate initialized near zero. Our Gradient Shadowing gate is a ReZero-style scalar; our contribution is not the gate itself but (i) its role in making \emph{surgery} exact rather than training stable, (ii) the explicit saddle-point analysis of the zero-weights-plus-zero-gate configuration, and (iii) its integration with a reactive invalidation engine that makes the surgery local.

\paragraph{Dynamic graph engines.} PyTorch \cite{pytorch} is define-by-run; JAX \cite{jax} traces pure functions. Neither maintains a persistent reactive graph across steps: validity of cached computation is not a first-class notion. NeuroDSL's design is closer to incremental computation systems, applied to differentiable programs.

\section{Preliminaries}
\label{sec:prelim}

\begin{definition}[Computational graph]
A computational graph is a tuple $\G = (\V, \E, \mathrm{op}, \theta)$ where $(\V, \E)$ is a finite DAG, $\mathrm{op}(v)$ assigns to each non-input node an operator, and $\theta(v)$ its (possibly empty) parameter set. Distinguished subsets $\V_{\mathrm{in}} \subset \V$ (sources) and a node $v_{\mathrm{out}}$ (output) induce, by composition along the topological order, a function $F_{\G} : \mathcal{X} \to \mathcal{Y}$.
\end{definition}

\begin{definition}[Grafting]
Let $e = (u, w) \in \E$ be an edge of $\G$ and $\Hb$ a computational graph with a single input and a single output, realizing a function $F_{\Hb}$. The grafting $\Surg(\G, e, \Hb)$ produces $\Gp$ by removing $e$ and adding edges $(u, \mathrm{in}(\Hb))$ and $(\mathrm{out}(\Hb), w)$, i.e., the value flowing along $e$ now passes through $\Hb$.
\end{definition}

Defining insertion on an \emph{edge} (rather than "at a node") removes any ambiguity about which consumers of $u$ are rerouted: exactly those along $e$.

\begin{definition}[Functional exactness]
$\Surg(\G, e, \Hb)$ is \emph{functionally exact} iff $F_{\Gp}(x) = F_{\G}(x)$ for all $x \in \mathcal{X}$. It is \emph{bit-exact} on a set $X_0 \subseteq \mathcal{X}$ of floating-point inputs iff the equality holds at the level of binary representations for all $x \in X_0$ under a fixed evaluation order.
\end{definition}

Note that functional exactness over the reals does not imply bit-exactness in floating-point arithmetic, nor conversely; Section~\ref{sec:exact} treats the two separately.

\section{Exact Grafting of Residual Blocks}
\label{sec:exact}

We instantiate $\Hb$ as a pre-normalization Transformer block in the Llama family \cite{llama}. Its exact sequential structure --- which we state faithfully, since the theorem depends on it --- is:
\begin{align}
    h &= x + \mathrm{Attn}\big(\mathrm{RMSNorm}(x);\, \theta_{\mathrm{attn}}\big), \label{eq:block1}\\
    H(x) &= h + \mathrm{MLP}\big(\mathrm{RMSNorm}(h);\, \theta_{\mathrm{mlp}}\big), \label{eq:block2}
\end{align}
where $\mathrm{RMSNorm}$ is root-mean-square normalization \cite{rmsnorm}, $\mathrm{Attn}$ ends with an output projection $W_O$, and $\mathrm{MLP}$ ends with a down projection $W_{\mathrm{down}}$.

\begin{assumption}[No output bias]
\label{ass:nobias}
The projections $W_O$ and $W_{\mathrm{down}}$ carry no bias term. (This holds for the Llama architecture \cite{llama}; if biases were present, they would additionally have to be initialized to zero.)
\end{assumption}

\begin{theorem}[Identity morphism]
\label{thm:identity}
Under Assumption~\ref{ass:nobias}, if at insertion time $t_{\mathrm{ins}}$ the parameters satisfy $W_O = 0$ and $W_{\mathrm{down}} = 0$ (all other parameters arbitrary), then $H = \mathrm{Id}$ over the reals, and consequently $\Surg(\G, e, \Hb)$ is functionally exact.
\end{theorem}

\begin{proof}
Since $\mathrm{Attn}$ terminates in the linear map $W_O$ with no bias, $W_O = 0$ implies $\mathrm{Attn}(\cdot) \equiv 0$ regardless of the attention scores; by \eqref{eq:block1}, $h = x$. Likewise $W_{\mathrm{down}} = 0$ implies $\mathrm{MLP}(\cdot) \equiv 0$, hence $H(x) = h = x$ by \eqref{eq:block2}. Grafting on $e = (u,w)$ replaces the value $x_e$ flowing along $e$ by $H(x_e) = x_e$; every node of $\Gp$ therefore computes the same real value as in $\G$, and $F_{\Gp} = F_{\G}$.
\end{proof}

Exactness over the reals is not yet bit-exactness: the graft introduces \emph{new floating-point operations} ($x + 0$, and the internal computation of the block). The following proposition states precisely when they are harmless.

\begin{proposition}[Qualified bit-exactness]
\label{prop:bitexact}
Work in IEEE-754 arithmetic with round-to-nearest. Under the hypotheses of Theorem~\ref{thm:identity}, suppose additionally that on input $x$:
\begin{enumerate}
    \item all activations internal to $\Hb$ upstream of $W_O$ and $W_{\mathrm{down}}$ are finite (no $\pm\infty$, no $\mathrm{NaN}$); and
    \item no component of the residual stream equals $-0.0$.
\end{enumerate}
Then the graft is bit-exact: every downstream node of $\Gp$ holds the same binary representation as in $\G$.
\end{proposition}

\begin{proof}
By hypothesis (1), each product inside the zero projections is of the form $0 \times a$ with $a$ finite, which is an exact $\pm 0.0$, and their sum is $\pm 0.0$; hence the branch outputs an exact zero vector. (Hypothesis (1) is necessary: $0 \times \infty = \mathrm{NaN}$, which would propagate.) The only remaining new operations are the residual additions $x_i + 0.0$. In round-to-nearest, $x_i + 0.0 = x_i$ bit-for-bit for every finite $x_i \neq -0.0$; the single exception is $-0.0 + 0.0 = +0.0$, excluded by hypothesis (2). All downstream nodes then receive bit-identical inputs and, computed by the same kernels in the same order, produce bit-identical outputs.
\end{proof}

\begin{remark}
Without hypothesis (2) the graft remains functionally exact (since $-0.0 = +0.0$ as real numbers) but not bit-exact; without hypothesis (1) it is not even functionally exact in floating point, as a $\mathrm{NaN}$ contaminates the residual stream. Empirically, hypothesis (2) is checkable in $O(n)$ before surgery, and hypothesis (1) holds whenever the incoming activations are finite and the randomly initialized internal weights are of moderate scale. The claim of bit-exactness in this paper is always understood under these hypotheses.
\end{remark}

\section{Topological Stability and Locality}
\label{sec:locality}

Surgery must leave the cached state of unaffected nodes intact. We formalize the invalidation procedure as Algorithm~\ref{alg:invalidate} and prove that it invalidates \emph{exactly} the downstream cone; the NeuroDSL routine \texttt{\_invalidate\_downstream!} implements this algorithm (implementation conformance is an engineering claim, verified by the test suite, not part of the proof).

\begin{algorithm}
\caption{\textsc{Invalidate}$(\G, s)$ --- reactive invalidation from a seed node $s$}
\label{alg:invalidate}
\begin{algorithmic}[1]
\State $Q \gets [s]$;\quad $\mathrm{seen} \gets \{s\}$
\While{$Q$ not empty}
    \State $v \gets \mathrm{pop}(Q)$
    \For{each $w$ with $(v, w) \in \E$}
        \If{$w \notin \mathrm{seen}$}
            \State $\mathrm{valid}(w) \gets \mathbf{false}$;\quad $\mathrm{seen} \gets \mathrm{seen} \cup \{w\}$;\quad $\mathrm{push}(Q, w)$
        \EndIf
    \EndFor
\EndWhile
\end{algorithmic}
\end{algorithm}

\begin{definition}[Downstream cone]
For $s \in \V$, let
\[
\V_s^{+} = \{\, v \in \V \mid \text{there is a path of length} \geq 1 \text{ from } s \text{ to } v \,\}.
\]
By convention $s \notin \V_s^{+}$ unless $s$ lies on a cycle, which the DAG property excludes; in our setting $s = \mathrm{out}(\Hb)$ is the freshly created node, which is born unvalued and needs no invalidation.
\end{definition}

\begin{theorem}[Structural locality]
\label{thm:locality}
On a DAG, Algorithm~\ref{alg:invalidate} terminates and the set of nodes it marks invalid is exactly $\V_s^{+}$. Its complexity is $O(|\V_s^{+}| + |\E_s^{+}|)$, where $\E_s^{+}$ is the set of edges internal to the cone.
\end{theorem}

\begin{proof}
\emph{Termination.} Each node enters $\mathrm{seen}$ at most once and is pushed at most once; $\V$ is finite.

\emph{Soundness} (marked $\Rightarrow$ in cone). We show by induction on the order of marking that every node marked invalid lies in $\V_s^{+}$. The first nodes marked are the immediate successors of $s$, which are in $\V_s^{+}$ via paths of length $1$. Inductively, a node $w$ is marked only when popped from the queue along an edge $(v, w)$ with $v$ previously marked or $v = s$; by the induction hypothesis there is a path $s \rightsquigarrow v$ (possibly empty if $v = s$), which extended by $(v,w)$ gives a path of length $\geq 1$ from $s$ to $w$. Hence $w \in \V_s^{+}$.

\emph{Completeness} (in cone $\Rightarrow$ marked). Let $w \in \V_s^{+}$ and let $s = v_0 \to v_1 \to \dots \to v_k = w$, $k \geq 1$, be a witnessing path. By induction on $i$: $v_0 = s \in \mathrm{seen}$ and is processed; if $v_i$ is processed, then when it is popped, its successor $v_{i+1}$ is either already in $\mathrm{seen}$ (hence was marked and pushed earlier) or is marked and pushed now. Either way $v_{i+1}$ is marked and eventually processed. Thus $v_k = w$ is marked.

\emph{Preservation.} A node $v \notin \V_s^{+}$ is never marked by soundness, so $\mathrm{valid}(v)$ is untouched: its cached value, and crucially its optimizer state, survive the surgery.

\emph{Complexity.} Each node in the cone is popped once and each internal edge scanned once.
\end{proof}

\begin{remark}[Consequence for training state]
Theorem~\ref{thm:locality} is what makes surgery cheap in a reactive engine: parameters outside the cone keep their values \emph{and} their Adam moments untouched, because state is owned by persistent nodes rather than by an external optimizer object indexed by parameter position. New parameters in $\Hb$ start with fresh (zero) moments, which is the standard and correct choice: their gradient history is genuinely empty. The interaction with learning-rate schedules is orthogonal and discussed in Section~\ref{sec:experiments}.
\end{remark}

Theorem~\ref{thm:locality} is the single-mutation statement this paper needs. The cost of entire \emph{workloads} of such mutations --- exhaustive site-by-site sweeps, sequences of persistent grafts, and the backward-pass mirror of the theorem --- is a separate accounting question with its own asymptotics, developed in a companion cost-accounting study on the same engine.

\section{Gradient Shadowing: Exactness with Plasticity}
\label{sec:shadow}

Theorem~\ref{thm:identity} buys exactness at a price: with $W_O = W_{\mathrm{down}} = 0$, all parameters \emph{upstream} of the zero projections receive a zero gradient (the chain rule factors through the zero matrices), so the block trains slowly at first --- only the zero projections themselves have non-zero gradients. Worse, a natural-looking "belt and braces" combination is fatal:

\begin{proposition}[Degenerate double-zero configuration]
\label{prop:saddle}
Consider the gated block $F_{\Hb}(x; \alpha) = x + \alpha R(x; \theta)$ with \emph{both} $\alpha = 0$ and the output projections of $R$ zero-initialized (so $R(x) \equiv 0$). Then at $t_{\mathrm{ins}}$:
\[
\frac{\partial \Loss}{\partial \alpha} = \Big\langle \frac{\partial \Loss}{\partial F_{\Hb}},\, R(x) \Big\rangle = 0
\qquad \text{and} \qquad
\nabla_{\theta} \Loss = \alpha \cdot \frac{\partial \Loss}{\partial F_{\Hb}} \cdot \frac{\partial R}{\partial \theta} = 0 .
\]
Moreover this configuration is a fixed point of gradient descent: both gradients remain identically zero at every subsequent step, so the block never trains.
\end{proposition}

\begin{proof}
At $t_{\mathrm{ins}}$, $R(x) = 0$ annihilates $\partial \Loss / \partial \alpha$ and $\alpha = 0$ annihilates $\nabla_\theta \Loss$. Gradient descent therefore leaves $(\alpha, \theta_{\mathrm{out}})$ unchanged, so the same two annihilations hold at the next step; by induction, forever. (Momentum-based optimizers without weight decay share the fixed point, since all moment estimates remain zero.)
\end{proof}

The two exactness mechanisms --- zero projections and zero gate --- are thus \emph{mutually exclusive}: one must choose exactly one. Gradient Shadowing chooses the gate, following ReZero \cite{rezero}:

\begin{definition}[Gradient Shadowing]
The grafted block is $F_{\Hb}(x; \alpha) = x + \alpha R(x; \theta)$, where $R$ is a residual branch (any composition of smooth primitives ending in linear projections) with $\theta$ drawn from a standard random initialization (so $R(\cdot;\theta) \not\equiv 0$ almost surely), and $\alpha \in \mathbb{R}$ is a learned scalar with $\alpha(t_{\mathrm{ins}}) = 0$.
\end{definition}

\begin{remark}[Parallel branch in the reference implementation]
\label{rem:parallel}
Propositions~\ref{prop:saddle} and \ref{prop:escape} are agnostic to $R$'s internal structure. The reference implementation (\texttt{graft\_shadow\_block!}) instantiates $R(x) = \mathrm{Attn}(N_1(x)) + \mathrm{MLP}(N_2(x))$ as a \emph{parallel} block (GPT-J style), not the sequential chain of \eqref{eq:block1}--\eqref{eq:block2}, for an experimental-design reason: in the sequential form, the MLP reads the attention output, so zeroing $W_O$ (needed for the comparison arms of Section~\ref{sec:experiments}) starves the MLP branch of both signal and gradient --- a coupling discovered empirically, where $\nabla_{W_{\mathrm{down}}}$ vanished exactly despite $W_{\mathrm{down}}$ never being zero-initialized. The parallel form lets each branch be disabled independently, so the three arms differ only in $(\alpha_0, W_{\mathrm{out}})$ and in nothing else. The sequential form of Theorem~\ref{thm:identity} remains the one used by the ungated surgery primitive (\texttt{insert\_block!}).
\end{remark}

Exactness at $t_{\mathrm{ins}}$ is immediate: $\alpha = 0$ gives $F_{\Hb} = \mathrm{Id}$, and Proposition~\ref{prop:bitexact} applies verbatim with "zero branch output" now produced by the multiplication $\alpha \cdot R(x) = 0 \times R(x)$, requiring the same finiteness hypothesis on $R(x)$. Plasticity is the content of the next proposition.

\begin{proposition}[Escape from initialization]
\label{prop:escape}
Let $g = \partial \Loss / \partial F_{\Hb}$ be the incoming gradient at the graft site on the first post-surgery batch. Then at $t_{\mathrm{ins}}$:
\[
\frac{\partial \Loss}{\partial \alpha} = \langle g,\, R(x;\theta) \rangle,
\]
which is non-zero for almost every random draw of $\theta$ (with respect to any initialization distribution absolutely continuous w.r.t.\ Lebesgue measure), provided $g \neq 0$. After one SGD step with learning rate $\eta$, $\alpha_1 = -\eta \langle g, R(x;\theta)\rangle \neq 0$, and from the second step onward $\nabla_\theta \Loss = \alpha_1 \, g \, \partial R / \partial \theta$ is generically non-zero: the entire branch trains.
\end{proposition}

\begin{proof}
The expression for $\partial \Loss / \partial \alpha$ is the chain rule applied to $F_{\Hb} = x + \alpha R$. Fix the incoming gradient $g \neq 0$ delivered by the first post-surgery batch; the almost-sure quantification is over $\theta$ at this fixed $g$, not over $g$. The map $\theta \mapsto \langle g, R(x;\theta) \rangle$ is real-analytic in $\theta$ (composition of matrix products, softmax, RMSNorm and smooth activations). It is not identically zero in $\theta$: under standard architectures the achievable set $\{R(x;\theta) : \theta\}$ spans the output space (e.g., scaling the final output projection scales $R$ linearly, so the achievable set contains a spanning family of rays), and the given non-zero $g$ cannot be orthogonal to a spanning set --- some $\theta^\star$ therefore gives $\langle g, R(x;\theta^\star)\rangle \neq 0$. A non-trivial real-analytic function vanishes on a Lebesgue-null set; hence, for this fixed $g$, the gradient is non-zero for almost every draw of $\theta$. The one-step update and the subsequent non-vanishing of $\nabla_\theta \Loss$ follow by substitution.
\end{proof}

\begin{remark}[Alternative: Net2Net-style plasticity]
\label{rem:net2net-alt}
If a scalar gate is undesirable, the symmetric choice is also sound: keep $W_O = W_{\mathrm{down}} = 0$ with \emph{no} gate ($\alpha$ fixed at $1$). Then $\partial \Loss / \partial W_{\mathrm{down}}$ is proportional to the (generically non-zero) pre-projection activations, so the zero projections move at the first step, unlocking the upstream gradients at the second. This is the Fixup/Net2DeeperNet regime \cite{fixup, net2net}; it trades the extra parameter $\alpha$ for a one-step-slower "warm-up" of the upstream weights. What is \emph{not} sound is combining the two (Proposition~\ref{prop:saddle}).
\end{remark}

\begin{remark}[Role of the reactive engine]
Gradient Shadowing requires the graph to hold, at the same time, a forward function that is exactly the identity and a backward rule set that sees through it. In NeuroDSL this is natural: forward dispatch and gradient-rule registration are decoupled, and the surgery installs both atomically inside one invalidation epoch, so no intermediate (inconsistent) graph state is ever observable by a training step.
\end{remark}

\subsection{Undoing a Graft: the Structural Notion}
\label{sec:undo}

A mutable graph raises a question a static one cannot even pose: what does it mean to \emph{undo} a graft, and what does it cost?

\begin{remark}[Structural undo is cheap and exact]
\label{rem:structural-undo}
Once training has moved the gate away from zero, the \emph{function} the network computes has drifted from the identity map that removing the block would restore; quantifying that functional drift is deferred to the preliminary appendix (Proposition~\ref{prop:undo-distance}). The \emph{engine state} needed to reverse the mutation's bookkeeping is a separate, and cheap, question: a reactive engine need only log the pre-graft values of $\V_s^+$ (Theorem~\ref{thm:locality}) to restore the graph to its exact pre-mutation state --- the same mechanism already used to restore a corrupted baseline between patches in the companion interpretability line of work. This \emph{structural} undo costs $O(|\V_s^+|)$ in log size and is always available, cheap, and exact, even when the functional drift is already strictly positive and the \emph{functional} undo is not free. The two notions are genuinely different --- one concerns what has been learned, the other what the engine remembers --- and only a graph that is both mutable and reactive makes both of them definable at all.
\end{remark}

\section{Experimental Results}
\label{sec:experiments}

Every theorem and proposition in this paper is validated by a dedicated experiment on the reference implementation (NeuroDSL, Julia, Float32, CPU). All protocols fix a master seed \emph{before} any run and apply no post-hoc filtering; per-arm and per-depth seeds are derived deterministically from it. One protocol parameter was recalibrated once: the initial budget (vocabulary 256, 200 steps) produced too weak a learning signal to discriminate the arms of E2; it was changed to vocabulary 50 and 600 steps in a single decision, before analysis, and not revisited afterwards.

\subsection{E1: Bit-Exactness (Theorem~\ref{thm:identity}, Proposition~\ref{prop:bitexact})}

A 4-block Llama-style model ($d=128$, 4 heads) is evaluated on a held-out batch; a Gradient-Shadowing block is grafted between blocks 2 and 3; the logits are recomputed with no parameter update and compared at the binary level (\texttt{reinterpret(UInt32, $\cdot$)}). Crucially, the preconditions of Proposition~\ref{prop:bitexact} are \emph{checked}, not assumed: a scan of the residual stream at the graft site found $0/4096$ components equal to $-0.0$, and $0/4096$ non-finite values in $R(x)$. Under these verified preconditions, all $1600$ logits are bit-identical ($0$ mismatches, $\max|\Delta| = 0.0$) --- strict binary equality, a stronger statement than the approximate function preservation reported by Net2Net or bert2BERT.

\subsection{E2: Plasticity and the Saddle Point (Propositions~\ref{prop:saddle}, \ref{prop:escape})}

Three arms share the same initial model, data stream, and random draw of $\theta$; they differ \emph{only} in $(\alpha_0, \text{zero-init of } W_{\mathrm{out}})$, made possible by the parallel branch of Remark~\ref{rem:parallel}. Training runs 600 SGD steps; Table~\ref{tab:e2} reports the grafted block's gradient signals.

\begin{table}[htbp]
\centering
\caption{E2 --- three-arm plasticity experiment (600 steps, shared master seed). "First step" columns give the first optimizer step at which the quantity becomes non-zero; "never" means it remained \emph{exactly} zero ($= 0$, no tolerance) for the entire run.}
\label{tab:e2}
\small
\setlength{\tabcolsep}{4pt}
\begin{tabular}{l c c c c c c}
\toprule
 & & & $\partial\Loss/\partial\alpha$ & \multicolumn{2}{c}{\textbf{First step}} & \\
\cmidrule(lr){5-6}
\textbf{Arm} & $\alpha_0$ & $W_{\mathrm{out}}$ & \textbf{at} $t{=}1$ & $|\alpha| > 0$ & $\|\nabla_\theta\| > 0$ & $\max\|\nabla_\theta\|$ \\
\midrule
(a) Grad.\ Shadowing & $0$ & random & $1.64 \times 10^{-2}$ & 1 & 2 & $3.28 \times 10^{-2}$ \\
(b) Net2Net-style    & $1$ & zero   & $0$ ($R \equiv 0$)     & --- & 1 & $1.65 \times 10^{0}$ \\
(c) Degenerate       & $0$ & zero   & $0$                    & never & never & $0$ (exact) \\
\bottomrule
\end{tabular}
\end{table}

Three predictions are confirmed at the resolution of a single optimizer step. \emph{Arm (a)}: $\partial\Loss/\partial\alpha = 1.64\times 10^{-2} \neq 0$ at the very first backward (Proposition~\ref{prop:escape}); $\alpha$ moves at step 1; $\nabla_\theta$ becomes non-zero at step 2, once $\alpha \neq 0$ reopens the chain rule --- the exact two-step escape the proposition derives. \emph{Arm (b)} is the mirror image predicted by the Net2Net remark: $\partial\Loss/\partial\alpha = 0$ at step 1 (since $R(x) \equiv 0$) while $\nabla_\theta \neq 0$ immediately (the zero projections themselves receive gradient), unlocking the upstream weights one step later. \emph{Arm (c)}: over all 600 steps, $\partial\Loss/\partial\alpha$, $\|\nabla_\theta\|$, and $|\alpha|$ remain \emph{identically} zero --- an empirical confirmation of an impossibility result (Proposition~\ref{prop:saddle}), at exact-zero resolution rather than within a tolerance. The host network's own loss still decreases in arm (c) (its pre-existing parameters keep training); the frozen quantity is the grafted block, not the model --- precisely the failure mode Proposition~\ref{prop:saddle} warns is silent.

Figure~\ref{fig:e2} plots the full 600-step trajectories behind Table~\ref{tab:e2}. The loss panel makes the silence of the failure mode visible: the three arms are nearly indistinguishable in loss, yet the gate and gradient panels show arm (c) exactly flat at zero throughout, while arms (a) and (b) each escape at the step the corresponding proposition predicts.

\begin{figure}[htbp]
\centering
\includegraphics[width=\linewidth]{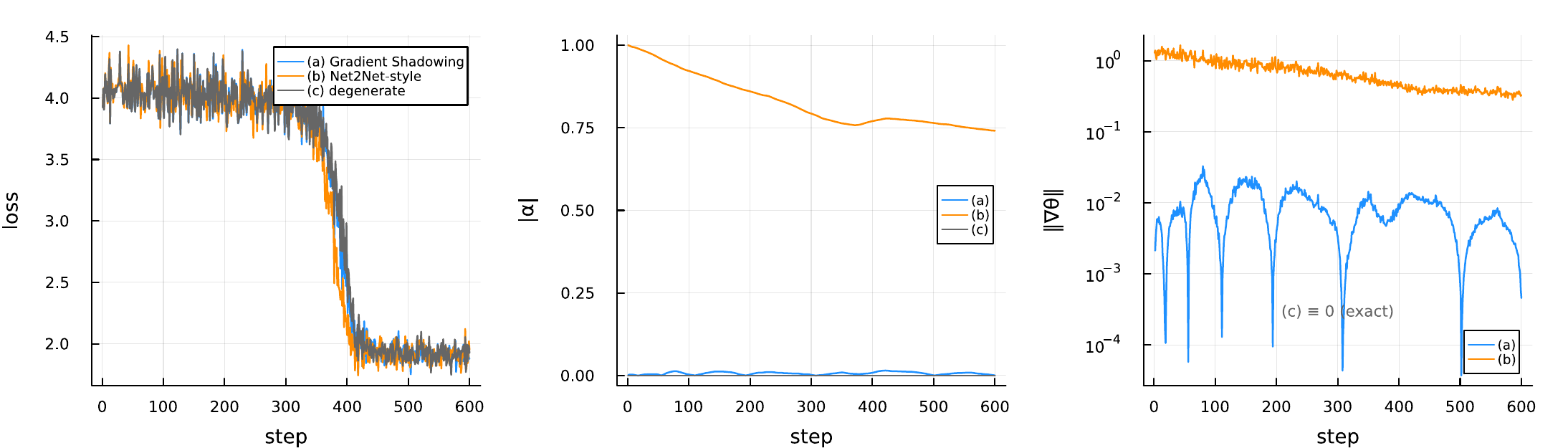}
\caption{E2 dynamics for the three arms over 600 steps: loss (left), gate magnitude $|\alpha|$ (center), and branch gradient norm $\|\nabla_\theta\|$ on a log scale (right). Arm (c) is exactly zero throughout in both the gate and gradient panels and is annotated rather than plotted, since an exact zero has no representation on a logarithmic axis.}
\label{fig:e2}
\end{figure}

An honest magnitude caveat: in arm (a), $|\alpha|$ remains small throughout the run ($\sim 10^{-3}$, final $2.7\times 10^{-4}$), so while the escape is mathematically confirmed, the block's \emph{contribution} to the function stays modest at this scale and duration --- the concern anticipated in Section~\ref{sec:limitations} about the unbounded-below magnitude of $\partial\Loss/\partial\alpha$, now observed. Since E2 uses one run per arm (by design: the arms are a controlled comparison, not a distribution estimate), we make no claim about the ordering of final losses across arms.

\subsection{E3: Surgery Cost vs.\ Insertion Depth (Theorem~\ref{thm:locality})}

On an 8-block model (412 graph nodes), a graft is performed after block $k$ for $k = 1, \dots, 7$; each sample uses a fresh, fully-evaluated model, and the cost is split into \emph{graft} (block construction + rewiring + invalidation, no evaluation) and \emph{recompute} (the first demand-driven forward, which re-executes exactly the invalidated cone). Timings are trimmed medians over 15 repetitions; depths are measured in shuffled order within each repetition --- a fixed order was found to inflate the first-measured depth's graft time by $\sim 2.6\times$ (a position artifact, removed by the shuffle, not a property of the engine).

\begin{table}[htbp]
\centering
\caption{E3 --- surgery cost vs.\ insertion depth $k$ (8-block model, 412 nodes, trimmed median over 15 shuffled repetitions).}
\label{tab:e3}
\begin{tabular}{c c c c c}
\toprule
$k$ & \textbf{Cone} $|\V_k^{+}|$ & \textbf{Graft (ms)} & \textbf{Recompute (ms)} & \textbf{Total (ms)} \\
\midrule
1 & 289 (70.1\%) & 0.745 & 14.503 & 15.248 \\
2 & 248 (60.2\%) & 0.745 & 12.626 & 13.371 \\
3 & 207 (50.2\%) & 0.754 & 11.072 & 11.826 \\
4 & 166 (40.3\%) & 0.745 & 9.067  & 9.812  \\
5 & 125 (30.3\%) & 0.741 & 7.660  & 8.401  \\
6 & 84 \ (20.4\%) & 0.750 & 5.590  & 6.340  \\
7 & 43 \ (10.4\%) & 0.751 & 4.178  & 4.929  \\
\bottomrule
\end{tabular}
\end{table}

Three observations. First, recomputation cost tracks cone size with $r = 0.9992$ and decreases strictly monotonically --- the direct validation of Theorem~\ref{thm:locality}: cost is governed by $|\V_k^{+}|$, not by the total graph size. Second, the graft-plus-invalidation bookkeeping is \emph{constant} ($0.741$--$0.754$~ms) across all depths, independent of cone size --- the signature of an $O(1)$-per-edge consumers index rather than a per-node rule scan; the surgery itself costs less than one fifth of even the smallest recomputation. Third, a linear fit gives $\text{recompute} \approx 0.042\,\text{ms/node} \times |\V_k^{+}| + 2.24\,\text{ms}$: the non-zero intercept is the $O(|\V|)$ walk of the cached topological order inside the demand-driven evaluator, paid even when only 43 nodes need recomputation. This residual interpreter overhead is exactly what a graph-compilation layer (flat precomputed execution order with $O(1)$ dirty marking) is designed to eliminate, and we leave its removal to that companion line of work.

Figure~\ref{fig:e3} plots recompute and graft cost against cone size directly: the linear trend and the flat bookkeeping line are both visible at a glance, and the fitted intercept is the $2.24$~ms evaluator overhead discussed above.

\begin{figure}[htbp]
\centering
\includegraphics[width=.75\linewidth]{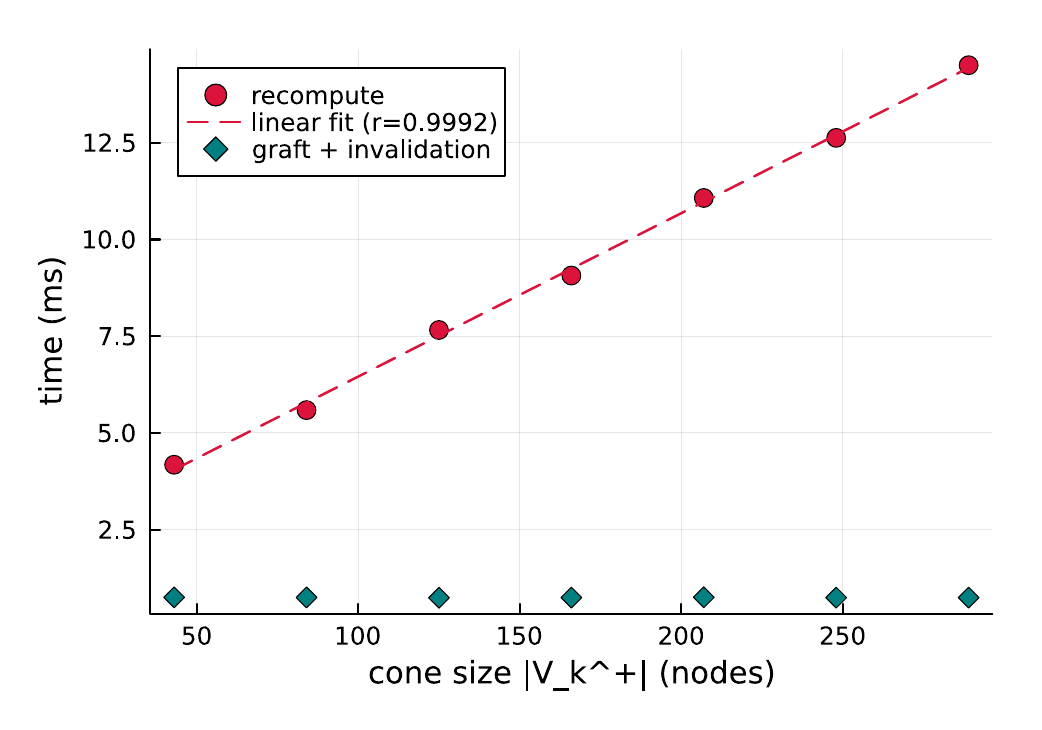}
\caption{Surgery cost vs.\ downstream cone size $|\V_k^{+}|$ on the 8-block model of Table~\ref{tab:e3} ($k = 1,\dots,7$). Recompute time (circles) grows linearly with cone size ($r=0.9992$); graft-plus-invalidation bookkeeping (diamonds) is flat at $\sim$0.75~ms regardless of cone size.}
\label{fig:e3}
\end{figure}

\subsection{F3: Continuity Across a Real Process Restart}

Surgery is only useful if the mutated model, including its optimizer state, survives the full lifecycle of real training --- interruption included. Three independent Julia \emph{processes} are compared: a control run of 500 steps with a graft at step 300; a run A executing the first 400 steps (graft included) then serializing graph and AdamW state and terminating; and a fresh process B that reloads everything from disk (never re-invoking the model constructor) and runs the remaining 100 steps. The concatenated A{+}B loss sequence is \emph{bit-identical} to the control on all 500 steps. Falsifiability check: deliberately discarding the loaded AdamW moments in B makes the run diverge from the control at exactly step 402 --- the first step at which corrupted optimizer state can affect a weight update --- confirming the comparison is sensitive to precisely the state whose preservation is claimed. Figure~\ref{fig:f3} shows all three loss curves: the control and the correctly-resumed run overlap exactly, while the broken-resume run visibly departs right after the restart marker.

\begin{figure}[htbp]
\centering
\includegraphics[width=.75\linewidth]{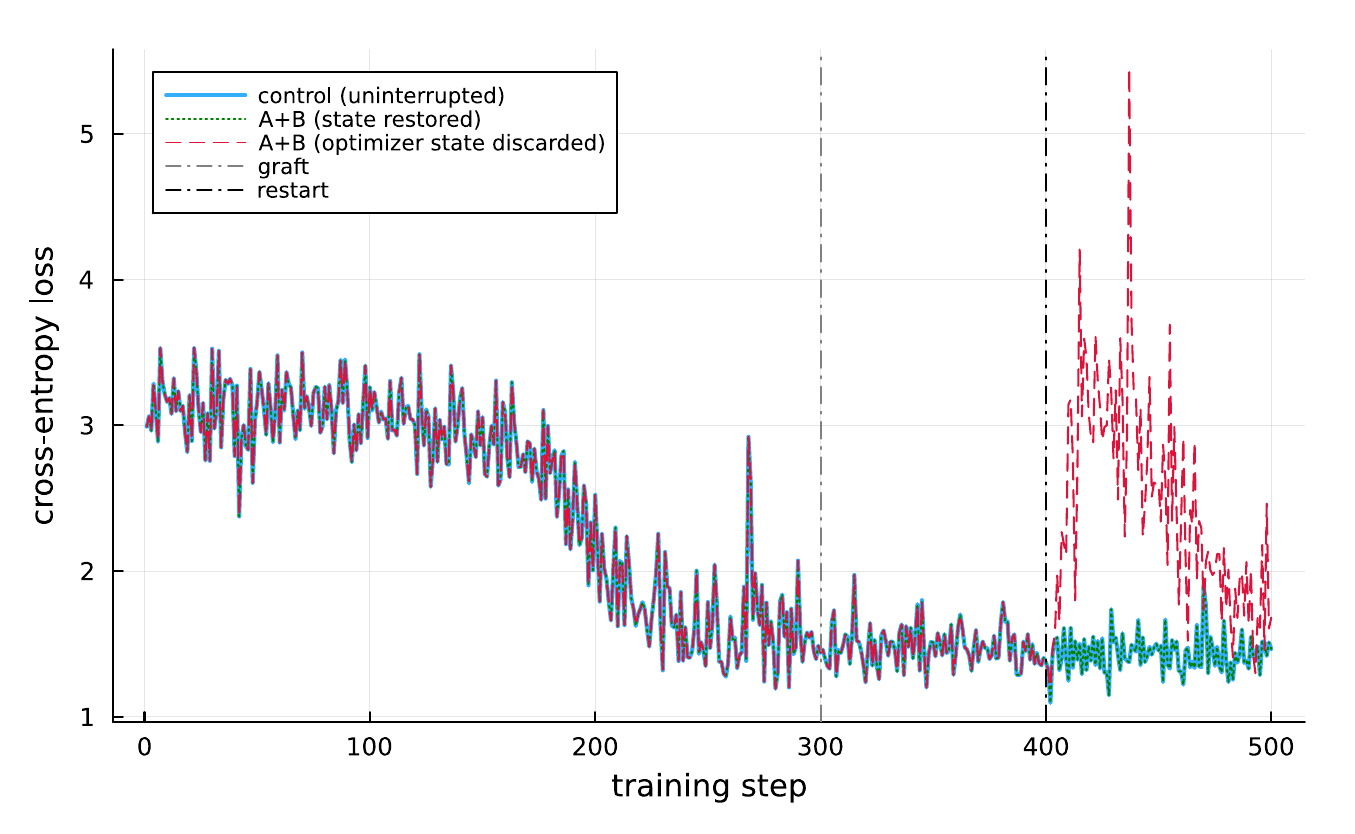}
\caption{Training loss through a graft (step 300) and a real process restart (step 400). The control run and the correctly-resumed A+B run are bit-identical for all 500 steps; a run that discards the loaded AdamW moments diverges from both at exactly step 402.}
\label{fig:f3}
\end{figure}

\subsection{Beyond Correctness: Utility at Equal Compute}

E1--E3 and F3 establish that surgery is exact, local, plastic, and durable; they say nothing about whether growing a model \emph{helps}. A pre-registered companion study on the induction task (commitment to report the outcome recorded before results were seen) found what an honest reading of the growth literature would predict: at equal \emph{step} count, the grafted model slightly trails training-from-scratch (mechanically --- it spends fewer steps at the final architecture); at equal \emph{compute} (the axis Net2Net and its successors actually claim, with the 3-vs-4-block per-step cost ratio measured at $0.76$ rather than assumed), a paired 20-seed comparison shows one statistically significant window: an early advantage for grafting at $16.7\%$ of the budget (95\% CI on the paired loss difference $[+0.16, +0.56]$, excluding zero), which fades to statistical indistinguishability from $33\%$ of the budget onward as both regimes saturate the task. We report this asymmetry as-is: the load-bearing claims of this paper are exactness and locality, not a training-efficiency advantage, and the task's early saturation limits what any utility comparison can resolve at this scale. Figure~\ref{fig:utility} shows the full paired-difference curve: the confidence interval clears zero only in the leftmost window.

\begin{figure}[htbp]
\centering
\includegraphics[width=.75\linewidth]{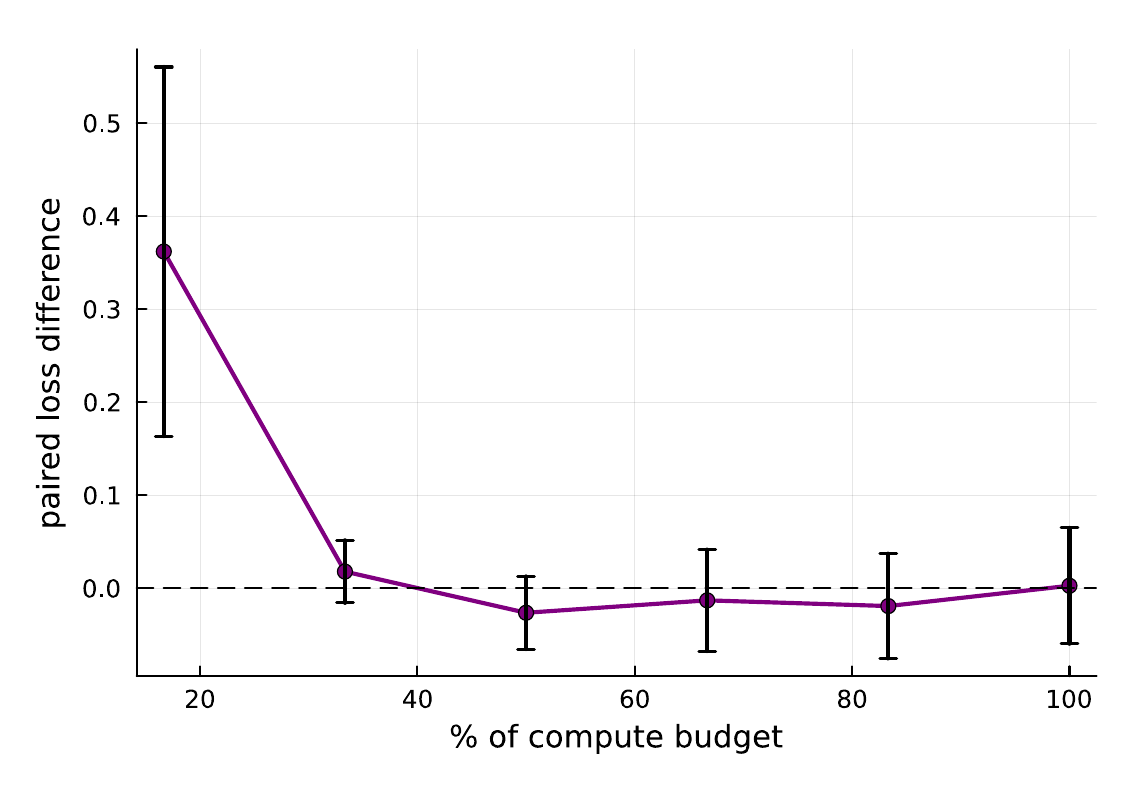}
\caption{Paired loss difference (loss from scratch $-$ loss grafted) at equal compute, mean $\pm$ 95\% CI over 20 seeds (normal approximation). Positive values favor grafting. The interval excludes zero only at $16.7\%$ of the budget.}
\label{fig:utility}
\end{figure}

\section{Limitations}
\label{sec:limitations}
Bit-exactness is guaranteed only under the stated IEEE-754 hypotheses and a fixed kernel and evaluation order; frameworks with non-deterministic reductions (e.g., atomics in some CUDA kernels) preserve functional but not bit-level equality, and our measurements are CPU\slash Float32 --- re-verifying E1 on GPU and in bf16 is mechanical but remains to be done. The almost-sure statement of Proposition~\ref{prop:escape} does not bound the \emph{magnitude} of $\partial\Loss/\partial\alpha$, and E2 makes this concrete: the escape occurs at step 1 as proved, but $|\alpha|$ stays near $10^{-3}$ over 600 steps on this task, so the grafted block's functional contribution ramps up slowly --- GradMax-style \cite{gradmax} initialization of $\theta$ could be combined with our gate to control this magnitude. That caveat concerns the gate alone; Table~\ref{tab:undo-factors} in the appendix measures the branch norm as well and finds it can move in the opposite direction, so a small $|\alpha|$ should not by itself be read as functional proximity to the identity. E2 uses a single controlled run per arm; its claims are about gradient signals (exact zeros and first-nonzero steps), not about loss distributions. The utility study is confined to a toy task that saturates early. The appendix's timing-dependence observation (Section~\ref{sec:e6}) is a single controlled comparison at one graft site on one task --- a falsifiable first data point, not a rate law --- and the exponential-relaxation ansatz of Definition~\ref{def:relaxation} is a fitted description, not a derived consequence of the paper's theorems; likewise E8 (Section~\ref{sec:e8}) isolates one mechanism on one seed, one site, and one task, and the reversed ordering it finds under warm-start is itself a single-run observation. Finally, our locality result concerns validity propagation; memory locality (buffer reuse across the surgery) is engine-specific and out of scope, and the measured $\sim$2.2~ms evaluator intercept in E3 bounds how far locality alone can reduce small-cone surgery cost without a compilation layer.

\section{Conclusion}
Network surgery becomes exact when the graft is an identity morphism, and remains plastic when exactness is carried by a single zero-initialized gate over a randomly initialized branch --- never by both mechanisms at once, on pain of an inescapable saddle point. A reactive graph engine makes the operation local: exactly the downstream cone is recomputed, and all other nodes keep their values and optimizer state. Every one of these statements now has a measurement attached to its proof: bit-exact grafting with verified preconditions (0/1600 mismatches), a two-step escape confirmed at the resolution of individual optimizer steps, a saddle point whose gradients stay at exact zero for 600 steps, locality with $r = 0.9992$ against cone size over constant $\sim$0.75~ms bookkeeping, and training that resumes bit-identically across a real process restart. Together, these results turn "growing a network mid-training" from a fragile engineering trick into an operation with a specification, a proof, and a passing test suite.

\appendix

\section{Appendix: Preliminary Observations on Post-Insertion Gate Dynamics}
\label{sec:appendix-gate}

\emph{Status of this appendix.} Everything in this appendix is exploratory: one seed, one graft site, one task, and a fitted (not derived) dynamical ansatz. These observations are reported at a deliberately lower evidentiary standard than the theorems and experiments of the main text --- as first, falsifiable data points on what happens to a grafted gate over the training that \emph{follows} insertion, not as established results.

\subsection{Functional Undo Distance}
\label{sec:functional-undo}

The following proposition quantifies the \emph{functional} counterpart of the structural undo of Section~\ref{sec:undo}: how far the grafted function has drifted from the identity map that a structural undo would restore.

\begin{proposition}[Functional undo distance]
\label{prop:undo-distance}
For the Gradient Shadowing construction $F_{\Hb}(x;\alpha) = x + \alpha R(x;\theta)$, the
functional distance between the currently grafted function and the identity map that a
\emph{structural} undo (removing the block, restoring the edge) would realize is, exactly and
without approximation,
\[
d(t) \;:=\; \mathbb{E}_x\big\|F_{\Hb}(x;\alpha_t) - x\big\| \;=\; |\alpha_t|\cdot \mathbb{E}_x\big\|R(x;\theta_t)\big\|.
\]
\end{proposition}

\begin{proof}
Immediate from the definition of $F_{\Hb}$: $F_{\Hb}(x;\alpha_t)-x = \alpha_t R(x;\theta_t)$, and
$\|\alpha_t R(x;\theta_t)\| = |\alpha_t|\cdot\|R(x;\theta_t)\|$ since $\alpha_t\in\mathbb{R}$.
\end{proof}

\begin{remark}[Irreversibility accumulates exactly at the rate the gate opens]
Proposition~\ref{prop:escape} guarantees $\alpha_t\neq0$ generically from $t=1$ onward, so
$d(t)>0$ from the same step: the graft is, in this precise functional sense, never perfectly
undoable again after its first update, however small $d(t)$ remains. This is not a shortcoming to
patch but the necessary price of Proposition~\ref{prop:escape}'s plasticity guarantee -- a graft
that could always be undone for free could not have moved away from the identity in the first
place. What is left to determine, and is not fixed by the theorem, is how large $d(t)$ becomes in
practice.
\end{remark}

The already-reported E2 trajectory (Table~\ref{tab:e2}, arm (a), $\S$\ref{sec:experiments}) bears
on this question, but only through a proxy: $|\alpha_t|$ is logged at every step, whereas the
second factor $\mathbb{E}_x\|R(x;\theta_t)\|$ of Proposition~\ref{prop:undo-distance} is not.
Read through that proxy, $|\alpha_t|$ rises from $0$ to order $10^{-2}$ by step $1$ (the escape),
oscillates at order $10^{-3}$ thereafter, and ends at $2.7\times10^{-4}$, which invites the
conclusion that $d(t)$ itself stayed small. The measurement reported next shows that this
inference does not survive contact with the missing factor, and we therefore do not draw it.

\paragraph{The gate alone does not measure the undo distance.}
A follow-up study logs \emph{both} factors of Proposition~\ref{prop:undo-distance} directly, on a
deliberately different setting: a character-level language model trained on real text, with a
two-head shadow block grafted by the same Gradient Shadowing construction ($\alpha_0=0$ over a
randomly initialized branch) and trained for $1000$ further steps, with the host network frozen so
that only the gate and the grafted branch move. The gate's own learning rate is held at $10^{-3}$
throughout while the learning rate on the branch parameters $\theta$ is swept across three orders
of magnitude, one run per rung. Table~\ref{tab:undo-factors} reports the two
factors at step $1000$. Branch magnitude is recorded as $\|R\|_{\mathrm{rms}} = \|R\|_F/\sqrt{n}$
over one fixed batch, computed identically for every row; it is proportional to
$\mathbb{E}_x\|R\|$ only up to a Jensen gap that closes when the per-example norms are equal, so
the final column is $d$ up to a factor not guaranteed identical across rows. What the argument
below rests on is the ordering, and effect sizes --- a $45\times$ growth, a $7.4\times$ spread ---
large enough that no plausible variation in that factor accounts for them.

\begin{table}[htbp]
\centering
\caption{Both factors of $d = |\alpha|\cdot\mathbb{E}_x\|R(x;\theta)\|$ at step $1000$, over a
learning-rate ladder on the grafted branch. All four runs start from the same branch
initialization ($\|R\|_{\mathrm{rms}} = 0.232$). The gate and the branch move in opposite
directions, and the resulting ordering by $d$ is the reverse of the ordering by $|\alpha|$.}
\label{tab:undo-factors}
\begin{tabular}{l c c c}
\toprule
Learning rate on $\theta$ & $|\alpha_{1000}|$ & $\|R\|$ growth & $d \propto |\alpha|\cdot\|R\|$ \\
\midrule
$10^{-3}$              & $0.108$ & $\times 45.1$ & $1.125$ \\
$10^{-4}$              & $0.383$ & $\times\phantom{0}4.8$ & $0.427$ \\
$10^{-5}$              & $0.644$ & $\times\phantom{0}1.4$ & $0.204$ \\
$2\times10^{-6}$ (+ weight decay $0.1$) & $0.626$ & $\times\phantom{0}1.0$ & $0.152$ \\
\bottomrule
\end{tabular}
\end{table}

Two observations follow. First, the unlogged factor is not slowly varying in general: with $\theta$
free, $\|R\|$ grows by a factor of $45$ over $1000$ steps, so it dominates $d$ rather than merely
rescaling it. Second, and more consequentially for how E2 should be read, the two factors are
\emph{anti}-correlated: freeing $\theta$ lets the branch grow while the gate closes to compensate,
and constraining $\theta$ pins $\|R\|$ near its initialization while the gate opens wide. The run
with the \emph{smallest} final gate --- the one a proxy reading would call the most nearly
undoable --- has the \emph{largest} undo distance, $7.4\times$ that of the run with the largest
gate. Ranking grafts by $|\alpha|$ alone therefore recovers the reverse of the intended order.

This does not overturn the E2 reading on its own run: that run is a different task at a different
learning rate over $600$ rather than $1000$ steps, and nothing measured here says its $d(t)$ was
large. What it removes is the ground for the general inference. Proposition~\ref{prop:undo-distance}
is an exact identity in two factors, and only measuring both settles how far a graft has drifted;
we accordingly log $\|R\|$ alongside $\alpha$ in all subsequent work. The usual caveats of this
appendix apply here in full --- one seed per rung, one graft site, one task.

\subsection{Absorption and Rejection Dynamics of a Grafted Gate}
\label{sec:absorption}

Propositions~\ref{prop:saddle} and \ref{prop:escape} settle whether the gate moves away from its
initial value; they say nothing about \emph{how}, over the training that follows. Unlike the rest
of this paper, that question has no closed form: $\alpha(t)$ for $t > t_{\mathrm{ins}}$ is the
output of a non-convex stochastic optimization, not a quantity a proof can pin down exactly.
What can be done honestly is to name the question precisely and test it empirically, in the same
spirit as the phenomenological (fitted, not derived) learning-curve model used in the companion
growth-schedule line of work.

\begin{definition}[Gate relaxation rate]
\label{def:relaxation}
Let $\alpha(t)$ be the trajectory of the gate for $t \geq t_{\mathrm{ins}}$ under either construction
of Section~\ref{sec:shadow}. Its \emph{relaxation rate} $\lambda$ is the coefficient of the
exponential ansatz
\[
\alpha(t) \;\approx\; \alpha_\infty + (\alpha_0 - \alpha_\infty)\, e^{-\lambda (t - t_{\mathrm{ins}})},
\]
fit by ordinary least squares on $\log|\alpha(t) - \alpha_\infty|$ (or, when $\alpha_\infty$ is not
independently known, on $\log|\alpha(t)|$ directly, as below). We call $\lambda > 0$
\emph{rejection} (the gate relaxes back toward a smaller magnitude than $\alpha_0$) and
$\lambda < 0$ \emph{absorption} (the gate grows in magnitude). The \emph{absorption time}
$\tau_{\mathrm{abs}}$ is $1/|\lambda|$, the trajectory's characteristic timescale.
\end{definition}

\begin{remark}[Not a theorem]
Definition~\ref{def:relaxation} names a fitting procedure, not a proven dynamical law: nothing in
this paper's proofs constrains $\alpha(t)$ beyond the sign of its first derivative at
$t_{\mathrm{ins}}$ (Proposition~\ref{prop:escape}). The exponential form is an ansatz, justified
only by its empirical fit quality, reported alongside every number derived from it below.
\end{remark}

This raises a natural training-time locality question --- the temporal analogue of the structural
locality of Theorem~\ref{thm:locality}: does $\lambda$ depend on \emph{when} the graft is inserted?
Two graphs identical in every respect except
$t_{\mathrm{ins}}$ isolate exactly this dependence, holding architecture, task, seed, and gate
construction fixed. We call this the \emph{timing-dependence hypothesis} and test it directly in
E6 below.

\begin{remark}[A candidate mechanism: optimizer-moment desynchronization]
\label{rem:desync-hypothesis}
A reactive graph makes a specific, checkable hypothesis available: the grafted branch's AdamW
second moment $v$ is initialized at zero at $t_{\mathrm{ins}}$ (the standard and correct choice per
the remark following Theorem~\ref{thm:locality}), while its sole downstream consumer's own $v$ is
whatever the host's training history has already accumulated. Since AdamW's effective step size
scales as $1/\sqrt{v+\epsilon}$, a near-zero $v$ gives the branch's freshly random parameters a
transiently oversized step relative to a downstream consumer whose $v$ is already large -- a
candidate mechanism for $\lambda>0$ that is falsifiable by a direct intervention: artificially
initializing the branch's $v$ to match its downstream consumer's, rather than to zero, and checking
whether $\lambda$ changes. E8 (Section~\ref{sec:e8}) performs exactly this intervention.
\end{remark}

\subsection{E6: Timing-Dependent Gate Dynamics of the Net2Net-Style Construction (Definition~\ref{def:relaxation})}
\label{sec:e6}

This experiment studies the \emph{Net2Net-style} construction of Remark~\ref{rem:net2net-alt} ---
arm (b) of Table~\ref{tab:e2}, $\alpha_0 = 1$ with zero output projections --- \emph{not} the
eponymous zero-initialized Gradient Shadowing gate of arm (a): the gate here starts at full
strength, and the question is whether the training that follows keeps it there. A small
character-level language model (TinyShakespeare, block size $64$, $d=64$, $2$ heads,
hidden $128$, $3$ Llama blocks) is trained from a fixed seed. A gated parallel block
(Remark~\ref{rem:parallel}) is grafted after the model's output at one of two insertion times,
$t_{\mathrm{ins}} \in \{500, 2500\}$ (``early'' / ``late''), with $\alpha_0 = 1$ and
$W_O = W_{\mathrm{down}} = 0$, so that the
branch starts at full gate strength but zero output, giving $\alpha(t)$ room to move in either
direction as the zero projections unlock. Both runs share architecture, task, corpus, and random
seed, and differ \emph{only} in $t_{\mathrm{ins}}$ -- the
controlled comparison Section~\ref{sec:absorption} asks for. Training continues for $1200$ steps
past the graft in both cases; every $100$ steps we log $\alpha(t)$ directly and measure
$A(t) = \Loss(\alpha{:=}0) - \Loss(\alpha_t)$ on a fixed held-out validation window via the same
cheap-ablation methodology as the companion patching line of work: patch $\alpha$ to $0$, evaluate,
restore $\alpha_t$, evaluate again -- no retraining, one forward pass per side.

\begin{table}[htbp]
\centering
\caption{E6 --- gate trajectory summary, early vs.\ late graft (single seed, single site).}
\label{tab:e6}
\begin{tabular}{l c c c c}
\toprule
\textbf{Arm} & $\alpha(t{=}100)$ & $\alpha(t{=}1200)$ & $\hat\lambda$ (fit, $R^2$) & $A(t)$ sign \\
\midrule
Early ($t_{\mathrm{ins}}{=}500$)  & $0.816$ & $0.524$ & $3.89\times10^{-4}$ ($0.986$) & crosses $-\!\to\!+$ near $t{\approx}650$ \\
Late ($t_{\mathrm{ins}}{=}2500$)  & $0.690$ & $0.369$ & $5.24\times10^{-4}$ ($0.961$) & negative throughout \\
\bottomrule
\end{tabular}
\end{table}

\begin{figure}[htbp]
\centering
\includegraphics[width=.75\linewidth]{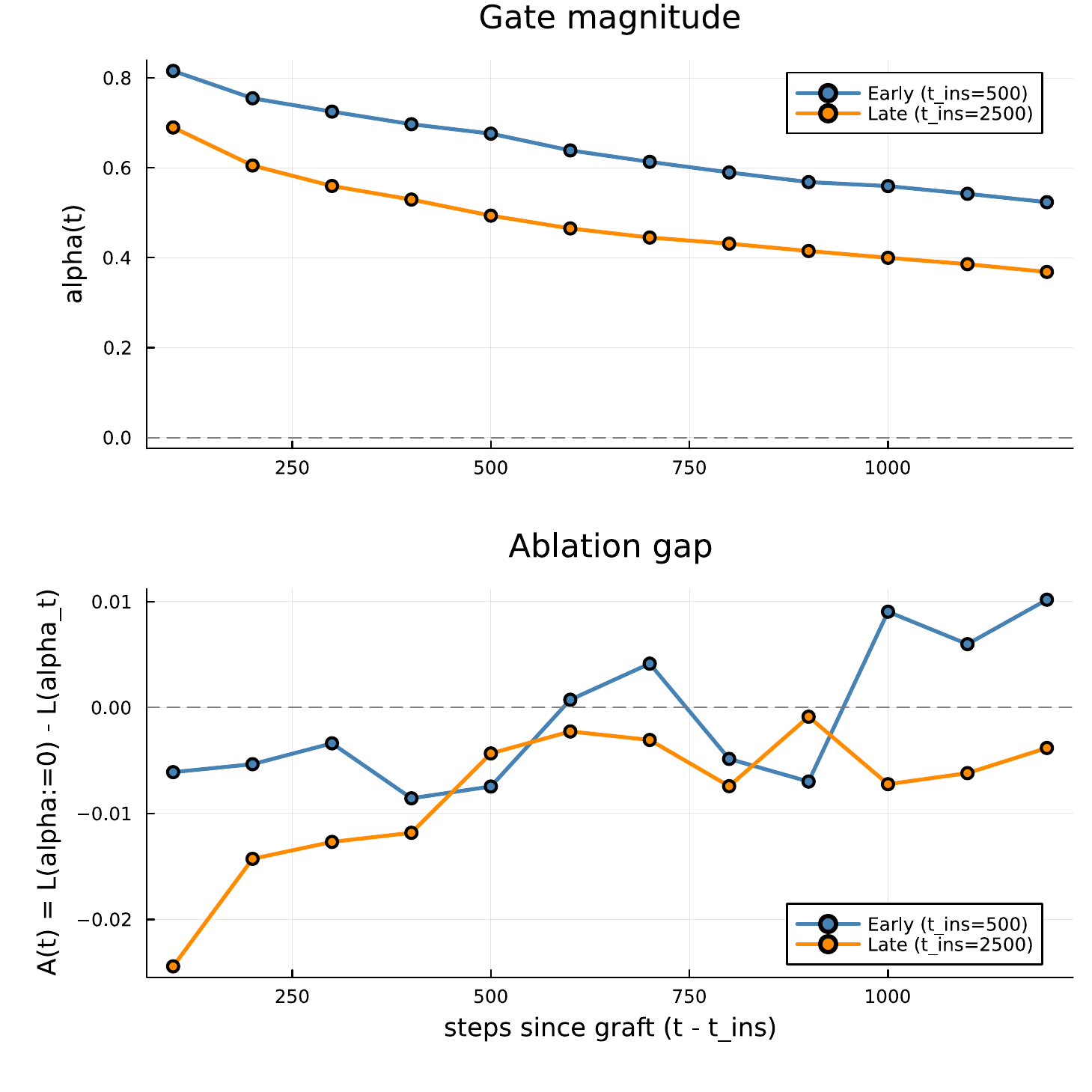}
\caption{E6 --- gate magnitude $\alpha(t)$ (top) and ablation gap $A(t)$ (bottom) for $1200$ steps
after an early ($t_{\mathrm{ins}}{=}500$) and a late ($t_{\mathrm{ins}}{=}2500$) graft. Both gates
decay monotonically (rejection, not absorption); the late graft decays faster and its $A(t)$ never
crosses zero, while the early graft's $A(t)$ turns positive from $t{\approx}650$ onward.}
\label{fig:e6}
\end{figure}

Both trajectories are well described by the exponential ansatz of Definition~\ref{def:relaxation}
($R^2 \geq 0.96$), and both give $\hat\lambda > 0$: neither graft is \emph{absorbed} in the sense
of the gate growing to accommodate the new capacity; both are \emph{rejected}, the optimizer
shrinking $\alpha$ back from its full-strength initialization as the randomly initialized branch
proves, on net, more disruptive than useful at this scale and duration -- the opposite of the naive
expectation that a freshly unlocked branch would be absorbed. The timing-dependence hypothesis is
confirmed, in this controlled single-seed comparison, in a specific direction: the late graft is
rejected $\approx 35\%$ faster ($\hat\lambda_{\mathrm{late}} / \hat\lambda_{\mathrm{early}} \approx 1.35$)
and its ablation gap $A(t)$ stays negative for the full $1200$-step window, while the early graft's
rejection decelerates and its contribution turns net-positive ($A(t) > 0$) from roughly $t \approx
650$ onward -- the branch that had more of the host network's own training left to interact with is
the one that is eventually kept, not the one that is discarded fastest.

We report this as a single controlled comparison, not a distributional claim: one seed, one graft
site, one task, exactly the caveat already attached to E2's arms (Table~\ref{tab:e2}), which this
experiment reuses without modification. The direction and the $\approx 35\%$ effect size should be
read as a first, falsifiable data point for the site/timing-dependence question, not as an
established rate law; a multi-seed, multi-site replication is the natural next step and is left to
future work. E8 (Section~\ref{sec:e8}) tests Remark~\ref{rem:desync-hypothesis}'s candidate
mechanism for this timing effect directly and finds it substantially, though not necessarily
completely, explains the ordering observed here -- read this section's $\approx 35\%$ figure
alongside that result, not in isolation.

\subsection{E8: Isolating the Mechanism -- an Optimizer-Moment Warm-Start Counterfactual}
\label{sec:e8}

Remark~\ref{rem:desync-hypothesis} proposes a mechanism for E6's rejection dynamic: the branch's
AdamW second moment $v$ starts at zero while its sole downstream consumer's ($\mathtt{lm\_head\_W}$,
the only parameterized node reading the graft's output) already carries a mature, non-zero $v$ from
the host's prior training. A first check of this mechanism, instrumenting E6's exact protocol to
log $v$ for the branch and for $\mathtt{lm\_head\_W}$ separately, confirmed real desynchronization
at $t_{\mathrm{ins}}$ (the branch's largest per-parameter $v$ is $100$--$200\times$ smaller than
$\mathtt{lm\_head\_W}$'s throughout the run) -- but the raw \emph{ratio} of this gap was, if
anything, \emph{smaller} for the late graft at matched steps-since-insertion, the opposite of what
would naively explain a \emph{faster} late-graft rejection; correlation alone was inconclusive; a
mean computed naively over all branch parameters was itself misleading here, reading as
numerically zero because $\sim\!90{,}000$ mostly-quiescent branch weights dilute the handful that
carry real signal -- resolved by reporting the per-parameter maximum instead, the quantity actually
governing AdamW's per-parameter step size.

The direct test: E6's protocol is rerun for both arms with one change -- immediately after the
graft, every branch parameter's $v$ is set not to zero but to $\mathrm{mean}(v_{\mathtt{lm\_head\_W}})$
at that instant (the host's own current moment estimate; $m_1$ is left at zero, the standard
choice, since only $v$ enters the effective-step-size mechanism of
Remark~\ref{rem:desync-hypothesis}).

\begin{table}[htbp]
\centering
\caption{E8 --- gate relaxation rate $\hat\lambda$ (Definition~\ref{def:relaxation}), cold- vs.\ warm-start second moment.}
\label{tab:e8}
\begin{tabular}{l c c c}
\toprule
\textbf{Arm} & $\hat\lambda$, cold-start (E6) & $\hat\lambda$, warm-start & Slowdown factor \\
\midrule
Early ($t_{\mathrm{ins}}{=}500$)  & $3.89\times10^{-4}$ & $5.62\times10^{-5}$ & $6.9\times$ \\
Late ($t_{\mathrm{ins}}{=}2500$)  & $5.24\times10^{-4}$ & $4.50\times10^{-5}$ & $11.6\times$ \\
\bottomrule
\end{tabular}
\end{table}

\begin{figure}[htbp]
\centering
\includegraphics[width=.75\linewidth]{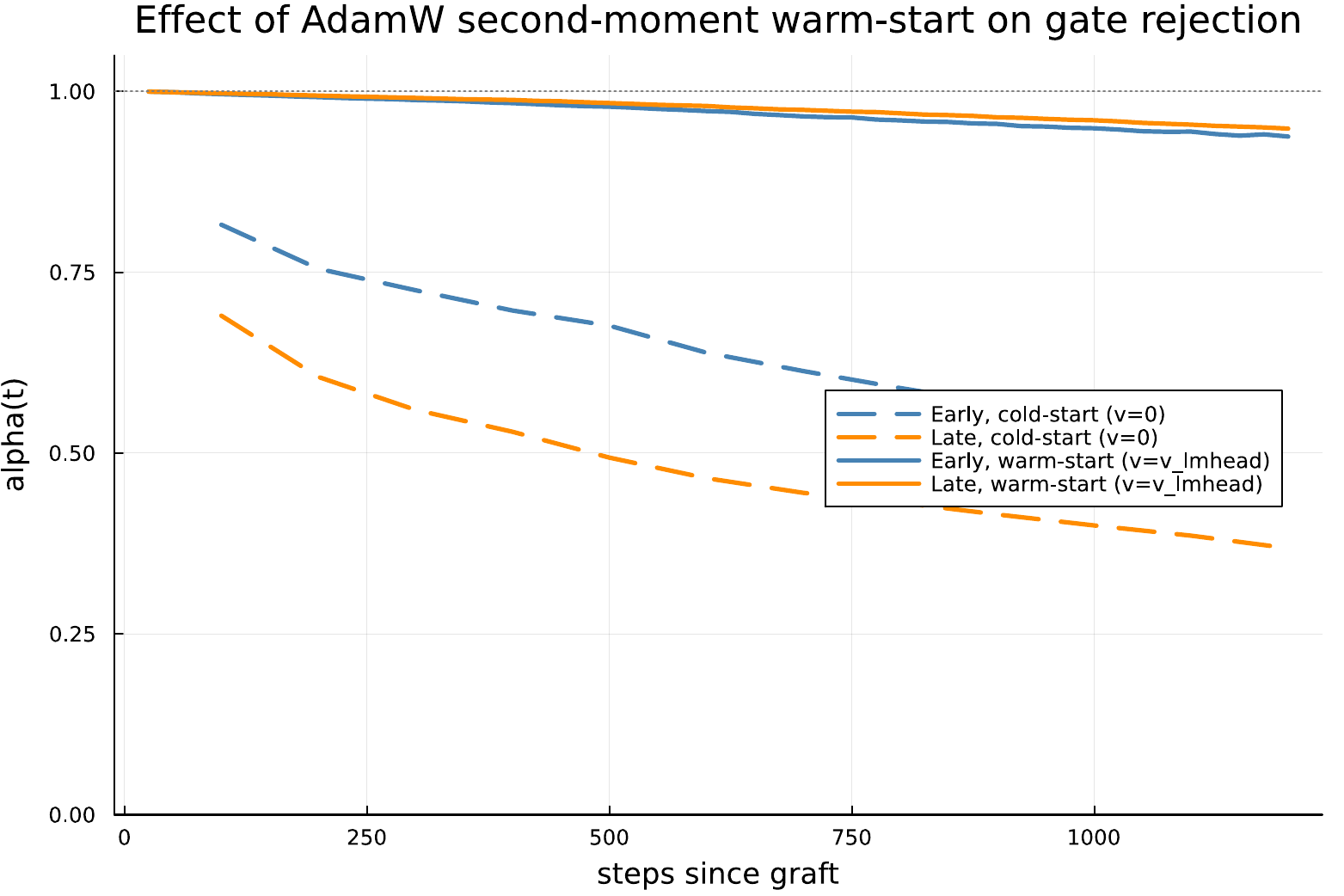}
\caption{E8 --- gate magnitude $\alpha(t)$, cold-start (E6, dashed) vs.\ warm-start (solid), both
arms. Warm-starting the branch's second moment to match its downstream consumer's collapses the
rejection rate by nearly an order of magnitude in both arms and reverses which arm rejects faster.}
\label{fig:e8}
\end{figure}

Two findings, both direct consequences of the intervention rather than of anything asserted in
advance. First, warm-starting slows rejection sharply in \emph{both} arms ($6.9\times$ and
$11.6\times$): the candidate mechanism is not merely present but load-bearing -- most of E6's
rejection magnitude is attributable to the cold-start effective-step-size mismatch, not to some
other property of a freshly random branch. Second, and unanticipated: \emph{the early/late ordering
reverses}. Under warm-start, $\hat\lambda_{\mathrm{late}} < \hat\lambda_{\mathrm{early}}$
($4.50\times10^{-5}$ vs.\ $5.62\times10^{-5}$, late $20\%$ \emph{slower}, not faster), whereas
cold-start gave the opposite ordering by $35\%$ (Section~\ref{sec:e6}). The most parsimonious
reading: $\mathtt{lm\_head\_W}$'s $v$ is itself larger at $t_{\mathrm{ins}}=2500$ than at
$t_{\mathrm{ins}}=500$ (more prior training), so the cold-start mismatch -- and hence the spurious
component of the timing effect -- is mechanically \emph{larger} for the late graft; once that
mismatch is equalized by warm-starting, a smaller, opposite-signed timing effect is revealed
underneath it. We do not claim to have isolated every contributor to this residual effect, only
that the dominant one in E6's cold-start reading was the optimizer-moment artifact, not a deeper
property of \emph{when} in training a graft occurs.

\begin{remark}[What this does and does not settle]
E8 is a single counterfactual intervention on one mechanism, on the same single-seed, single-site,
single-task setup as E6 -- it does not establish that optimizer-moment desynchronization is the
\emph{only} driver of gate rejection, nor that the reversed warm-start ordering is itself stable
across seeds. What it does establish, falsifiably and at this scale, is that a specific, named,
checkable mechanism (Remark~\ref{rem:desync-hypothesis}) accounts for enough of E6's effect size to
change its sign under a direct intervention -- exactly the kind of question a persistent, reactive
graph makes cheap to ask, since $v$ is already resident per-parameter state rather than something a
tape-based framework would need new machinery to intercept and rewrite mid-training.
\end{remark}


\begin{thebibliography}{14}
\bibitem{net2net} Chen, T., Goodfellow, I., Shlens, J. (2016). Net2Net: Accelerating Learning via Knowledge Transfer. \emph{ICLR}.
\bibitem{rezero} Bachlechner, T., Majumder, B.P., Mao, H., Cottrell, G., McAuley, J. (2021). ReZero is All You Need: Fast Convergence at Large Depth. \emph{UAI}.
\bibitem{fixup} Zhang, H., Dauphin, Y.N., Ma, T. (2019). Fixup Initialization: Residual Learning Without Normalization. \emph{ICLR}.
\bibitem{layerscale} Touvron, H., Cord, M., Sablayrolles, A., Synnaeve, G., J\'egou, H. (2021). Going Deeper with Image Transformers. \emph{ICCV}.
\bibitem{stacking} Gong, L., He, D., Li, Z., Qin, T., Wang, L., Liu, T.-Y. (2019). Efficient Training of BERT by Progressively Stacking. \emph{ICML}.
\bibitem{bert2bert} Chen, C., Yin, Y., Shang, L., Jiang, X., Qin, Y., Wang, F., Wang, Z., Chen, X., Liu, Z., Liu, Q. (2022). bert2BERT: Towards Reusable Pretrained Language Models. \emph{ACL}.
\bibitem{lemon} Wang, Y., et al. (2024). LEMON: Lossless Model Expansion. \emph{ICLR}.
\bibitem{gradmax} Evci, U., van Merri\"enboer, B., Unterthiner, T., Vladymyrov, M., Pedregosa, F. (2022). GradMax: Growing Neural Networks using Gradient Information. \emph{ICLR}.
\bibitem{frozensubstrate} Bochkov, A. (2025). Growing Transformers: Modular Composition and Layer-wise Expansion on a Frozen Substrate. \emph{arXiv:2507.07129}.
\bibitem{llama} Touvron, H., et al. (2023). LLaMA: Open and Efficient Foundation Language Models. \emph{arXiv:2302.13971}.
\bibitem{rmsnorm} Zhang, B., Sennrich, R. (2019). Root Mean Square Layer Normalization. \emph{NeurIPS}.
\bibitem{pytorch} Paszke, A., et al. (2019). PyTorch: An Imperative Style, High-Performance Deep Learning Library. \emph{NeurIPS}.
\bibitem{jax} Bradbury, J., et al. (2018). JAX: Composable Transformations of Python+NumPy Programs. \url{http://github.com/google/jax}.
\bibitem{neurodsl} Khemais, A. (2026). NeuroDSL: A Reactive Computational Graph Framework for Deep Learning in Julia. \emph{Preprint}. Reference implementation: \url{https://github.com/nevermind78/NeuroDSL}.
\end{thebibliography}
\end{document}